\pdfoutput=1
\documentclass[letterpaper, 10 pt, conference]{ieeeconf}  
\usepackage{ifxetex}
\ifxetex\else\usepackage[utf8]{inputenc}\fi

\IEEEoverridecommandlockouts                              

\usepackage{subcaption}
\usepackage{amssymb}
\usepackage{amsmath}
\usepackage{cases}
\usepackage{amsfonts}

\usepackage{textcomp}

\usepackage[capitalise]{cleveref} 
\crefname{section}{Sec.}{Secs.}
\Crefname{section}{Section}{Sections}
\crefname{figure}{Fig.}{Figs.}
\Crefname{figure}{Figure}{Figures}
\crefname{equation}{}{}
\Crefname{equation}{Equation}{Equations}

\usepackage{tikzscale}

\usepackage{rotating}

\usepackage[detect-all]{siunitx}
\usepackage{pgfplots}

\usepackage{tikz}
\usepgfplotslibrary{groupplots}
\usetikzlibrary{plotmarks}
\usetikzlibrary{positioning}
\usetikzlibrary{shapes,arrows, fit, positioning}
\usetikzlibrary{narrow}
\usepgfplotslibrary{external}
\usepackage{tikz-3dplot}
\usepackage{tikzscale}
\usepackage{centeredtikzarrowbulletheads}
\tikzset{>=latex}
\tikzset{
 font={\fontsize{10pt}{12}\selectfont}}
\pgfplotsset{every tick label/.append style={font=\small}}
\pgfplotsset{legend style={font=\small}}
\pgfplotsset{every axis label/.append style={font=\small}}

\title{\LARGE \bf
Temporally Coupled Dynamical Movement Primitives \\ in Cartesian Space}

\author{
\centering Martin Karlsson* \quad Anders Robertsson \quad Rolf Johansson
\thanks{* The authors work at the Department of Automatic Control, Lund University, PO Box 118,
SE-221 00 Lund, Sweden.\protect\\
{Martin.Karlsson@control.lth.se\protect\\
The research leading to these results has received funding from the Vinnova project \textit{Kirurgens Perspektiv}. The authors are members of the LCCC Linnaeus Center and the ELLIIT Excellence Center at Lund University.}}
}

\begin{document}
\newcommand{\cmt}[1]{{\color{red}{\textbf{Comment:} #1}}}

\newtheorem{proposition}{Proposition}
\newtheorem{theorem}{Theorem}
\newtheorem{lemma}{Lemma}
\newtheorem{remark}{Remark}
\newtheorem{definition}{Definition}


\maketitle
\thispagestyle{empty}
\pagestyle{empty}
\begin{abstract}
Control of robot orientation in Cartesian space implicates some difficulties, because the rotation group SO(3) is not contractible, and only globally contractible state spaces support continuous and globally asymptotically stable feedback control systems. In this paper, unit quaternions are used to represent orientations, and it is first shown that the unit quaternion set minus one single point is contractible. This is used to design a control system for temporally coupled dynamical movement primitives (DMPs) in Cartesian space. The functionality of the control system is verified experimentally on an industrial robot.
\end{abstract}

\section{Introduction}
Industrial robots typically work well for tasks where accurate position control is sufficient, and where work spaces and robot programs have been carefully prepared, so that hardware configurations can be foreseen a priori by robot programmers in each step of the tasks. Such preparation is very time consuming, and introduces high costs in terms of engineering work. Further, the arrangements are sensitive to variations, \textit{e.g.}, uncertainties in work object positions, small differences between individual work objects, \textit{etc.} This has prohibited the automation of a range of tasks, including seemingly repetitive ones such as assembly tasks and short-series production.

It would therefore be beneficial if the capabilities of robots to adapt to their surroundings could be improved. The framework of dynamical movement primitives (DMPs), used to model robot movement, has an emphasis on such adaptability \cite{ijspeert2013dynamical}. For instance, the time scale and goal position of a movement can be adjusted through one parameter each. The fundamentals of DMPs have been described in \cite{ijspeert2013dynamical}, and earlier versions have been introduced in \cite{schaal2000nonlinear,ijspeert2002humanoid,ijspeert2003learning}. DMPs have been used to modify robot movement based on moving targets in the context of object handover \cite{prada2014handover}, and based on demonstrations by humans \cite{karlsson2017autonomous,karlsson2017motion,chiara2018passivity,karlsson2019human}. In most of the previous research, it has been assumed that the robot configuration space is a real coordinate space, such as joint space or Cartesian position space; see, \textit{e.g.}, \cite{prada2014handover,karlsson2017autonomous,chiara2018passivity,papageorgiou2018sinc,yang2018learning}. However, in \cite{ude2014orientation} DMPs were formulated for orientation in Cartesian space.

Temporal coupling for DMPs enables robots to recover from unforeseen events, such as disturbances or detours based on sensor data. This concept was introduced in \cite{ijspeert2013dynamical}, was made practically realizable in \cite{karlsson2017dmp}, and proven exponentially stable in \cite{karlsson2018convergence}. However, these previous results are applicable only if the robot state space is Euclidean, which is not true for orientation in Cartesian space. Higher levels of robot control typically operate in Cartesian space, for instance to control the pose of a robot end-effector or an unmanned aerial vehicle.

In this paper, we therefore address the question of whether the control algorithm in \cite{karlsson2017dmp} could be extended also to incorporate orientations. Because a contractible state space is necessary for design and analysis of a continuous globally asymptotically stable control law (see \cref{sec:contractible}), we first investigate the contractibility properties of the quaternion set used to represent orientations. A space is contractible if and only if it is homotopy equivalent to a one-point space \cite{hatcher2002algebraic}, which intuitively means that the space can be deformed continuously to a single point; see, \textit{e.g.}, \cite{hatcher2002algebraic} for a definition of homotopy equivalence.

\subsection{Contribution}
This paper provides a control algorithm for DMPs with temporal coupling in Cartesian space. It extends our previous research in \cite{karlsson2017dmp,karlsson2018convergence} by including orientation in Cartesian space. Equivalently, it extends \cite{ude2014orientation} by including temporal coupling. Furthermore, it is shown that the quaternion set minus one single point is contractible, which is a necessary property for design of a continuous and globally asymptotically stable control algorithm. Finally, the theoretical results are verified experimentally on an ABB YuMi robot; see \cref{fig:yumi_gore_tex} and \cite{yumi}.

\begin{figure}
\centering
\includegraphics[width=0.98\columnwidth]{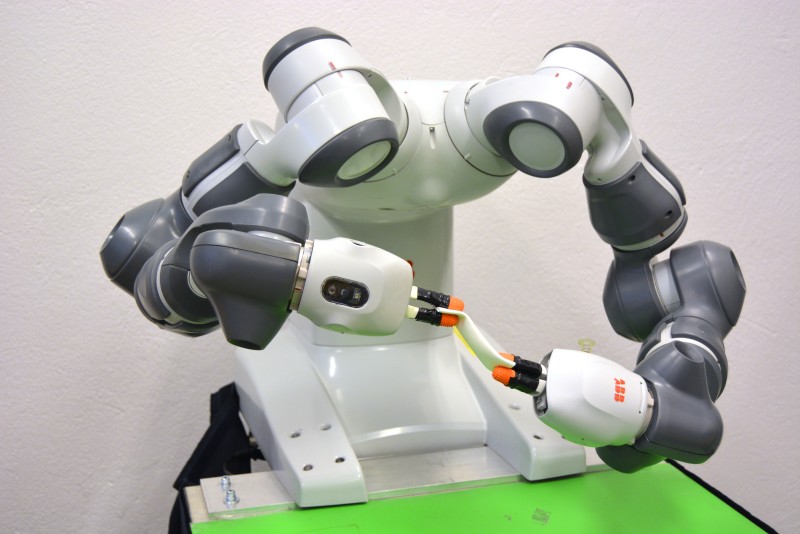}
\caption{The ABB YuMi robot \cite{yumi} used in the experiments.}
\label{fig:yumi_gore_tex}
\end{figure}

\pagebreak

\section{A Contractible Subset \\ of the Unit Quaternion Set}
\label{sec:contractible}
The fundamentals of mathematical topology and set theory are described in, \textit{e.g.}, \cite{hatcher2002algebraic,crossley2006essential,schwarz2013topology}. As noted in \cite{mayhew2011quaternion}, the rotation group SO(3) is not contractible, and therefore it is not possible for any continuous state-feedback control law to yield a globally asymptotically stable equilibrium point in SO(3) \cite{bhat2000topological,koditschek1988application}. Contractibility is also necessary to apply the contraction theory from \cite{lohmiller1998contraction}, as done in \cite{karlsson2018convergence}. In this paper, unit quaternions are used to parameterize SO(3). Similarly to SO(3), the unit quaternion set, $\mathbb{H}$, is not contractible. In this section however, is is shown that it is sufficient to remove one point from $\mathbb{H}$ to yield a contractible space. \cref{table:notation} lists some of the notation used in this paper.

\begin{table}
\begin{center}
\caption{Notation used in this paper. All quaternions represent orientations and are therefore of unit length. For such quaternions, the inverse is the same as the conjugate.} \label{table:notation}
\begin{tabular}{l l l}
Notation  &  & Description \\
\hline
$\mathbb{H}$ &   & Unit quaternion set \\
$\mathbb{S}^n$ &  $\in \mathbb{R}^{n+1}$ & Unit sphere of dimension $n$ \\
$y_a$ & $\in \mathbb{R}^3$ & Actual robot position \\
$g$ & $\in \mathbb{R}^3$ & Goal position \\
$y_c$ & $\in \mathbb{R}^3$ & Coupled robot position \\
$q_a$ & $\in \mathbb{H}$ & Actual robot orientation \\
$q_g$ & $\in \mathbb{H}$ & Goal orientation \\
$q_c$ & $\in \mathbb{H}$ & Coupled robot orientation \\
$\omega_c$ & $\in \mathbb{R}^3$ & Coupled angular velocity\\
$q_0$ & $\in \mathbb{H}$ & Initial robot orientation \\
$\mathfrak{h}$  & & Quaternion difference space \\
$d_{cg}$ & $\in \mathfrak{h}$  & Difference between $q_c$ and $q_g$   \\
$z, \omega_z$ & $\in \mathbb{R}^3$ & DMP states \\
$\alpha_z, \beta_z, k_v, k_p$ & $\in \mathbb{R}^+$ & Constant control coefficients \\
$\tau$ & $\in \mathbb{R}^+$ & Nominal DMP time constant \\
$\tau_a$ & $\in \mathbb{R}^+$ & Adaptive time parameter \\
$x$ & $\in \mathbb{R}^+$ & Phase variable \\
$\alpha_x, \alpha_e, k_c$ & $\in \mathbb{R}^+$ & Positive constants \\
$f(x)$ & $\in \mathbb{R}^6$  & Learnable virtual forcing term \\
$f_p(x), f_o(x)$ & $\in \mathbb{R}^3$ & Position and orientation components \\
$N_b$ & $\in \mathbb{Z}^+$ & Number of basis functions \\
$\Psi_j(x)$ & $\in \mathbb{R}^6$ & The $j$:th basis function vector \\
$w_j$ & $\in \mathbb{R}^6$ & The $j$:th weight vector \\
$e$ & $\in \mathbb{R}^3 \times \mathfrak{h} $ & Low-pass filtered pose error \\
$e_p$ & $\in \mathbb{R}^3$ & Position component of $e$ \\
$e_o$ & $\in \mathfrak{h}$ & Orientation component of $e$ \\
$\ddot{y}_r, \dot{\omega}_r$ & $\in \mathbb{R}^3$ & Reference robot acceleration \\
$\xi$ &  $ \in \mathbb{R}^{22} \times \mathfrak{h}^3$ & DMP state vector \\
$\bar{q}$ & $\in \mathbb{H}$ & Inverse of quaternion $q$ \\
$\simeq $ & & Homotopy equivalence \\
$\cong$ & & Homeomorphic relation \\
\hline
\end{tabular}
\end{center}
\end{table}

\subsection{Preliminary topology}
We will use that homeomorphism (defined in, \textit{e.g.}, \cite{crossley2006essential}) is a stronger relation than homotopy equivalence.
\begin{lemma}
\label{lm:homeomorphism}
If two spaces $X$ and $Y$ are homeomorphic, then they are homotopy equivalent.
\end{lemma}
\begin{proof}
See Lemma 6.11 in \cite{crossley2006essential}.
\end{proof}

\begin{lemma}
\label{lm:remove_point}
Assume that $X \cong Y$, with a homeomorphism $f: X \rightarrow Y$. Then $X$ minus a point $p \in X$, denoted $X\setminus p$, is homeomorphic to $Y\setminus f(p)$.
\end{lemma}
\begin{proof}
Consider the function $f_2: X\setminus p \rightarrow Y\setminus f(p)$, and let $f_2(x) = f(x) \hspace{2mm} \forall x \in X\setminus p$. It can be seen that $f_2$ is a restriction of $f$. Since a restriction of a homeomorphism is also a homeomorphism \cite{lehner1964discontinuous}, $f_2$ is a homeomorphism, and hence $X\setminus p \cong Y\setminus f(p)$.
\end{proof}

We will also use that homeomorphism preserves contractibility.
\begin{lemma}
\label{lm:contract_homeomorphism}
If $X \cong Y$, and $X$ is contractible, then $Y$ is also contractible.
\end{lemma}
\begin{proof}
Since $X \cong Y$, they are homotopy equivalent according to \cref{lm:homeomorphism}. In turn, $X$ is contractible and therefore homotopy equivalent to a one-point space. Hence $Y$ is also homotopy equivalent to a one-point space, and therefore contractible.
\end{proof}

\subsection{The quaternion set minus one point is contractible}
First, it will be shown that the unit sphere $\mathbb{S}^n$ (see \cref{def:unit_sphere}) minus a point is contractible. This will then be applied to $\mathbb{H}$, which is homeomorphic to $\mathbb{S}^3$ \cite{lavalle2006planning}.
\begin{definition}
\label{def:unit_sphere}
Let $n$ be a non-negative integer. The unit sphere with dimension $n$ is defined as
\begin{equation}
\mathbb{S}^n = \left\lbrace p \in \mathbb{R}^{n+1} \hspace{2mm} \mid \hspace{2mm}    \left\lVert p \right\rVert _2 = 1 \right\rbrace
\end{equation}
\end{definition}

\begin{theorem}
\label{thm:sphere_n}
The unit sphere $\mathbb{S}^n$ minus a point $p \in \mathbb{S}^n$, denoted $\mathbb{S}^n\setminus{p}$, is contractible. 
\end{theorem}

\begin{proof}
Consider first the case $n \geq 1$. There exists a mapping from $\mathbb{S}^n\setminus{p}$ to $\mathbb{R}^n$ called stereographic projection from $p$, which is a homeomorphism. Thus, $\mathbb{S}^n\setminus{p} \cong \mathbb{R}^n$ \cite{huggett2009topological,schwarz2013topology}. See \cref{fig:sphere_vs_plane} for a visualization of these spaces. Since $\mathbb{R}^n$ is a Euclidean space it is contractible, and it follows from \cref{lm:contract_homeomorphism} that $\mathbb{S}^n\setminus{p}$ is also contractible.

Consider now the case $n=0$. The sphere $\mathbb{S}^0$ consists of the pair of points $\{-1,1\}$ according to \cref{def:unit_sphere}. Thus $\mathbb{S}^0\setminus{p}$ consists of one point only, and homotopy equivalence with a one-point space is trivial. Hence $\mathbb{S}^0\setminus{p}$ is contractible.
\end{proof}

\begin{remark}
Albeit we consider unit spheres in this paper, it is not necessary to assume radius 1 in \cref{thm:sphere_n}. Further, it is arbitrary which point $p \in \mathbb{S}^n$ to remove.
\end{remark}

\begin{figure}
	\input{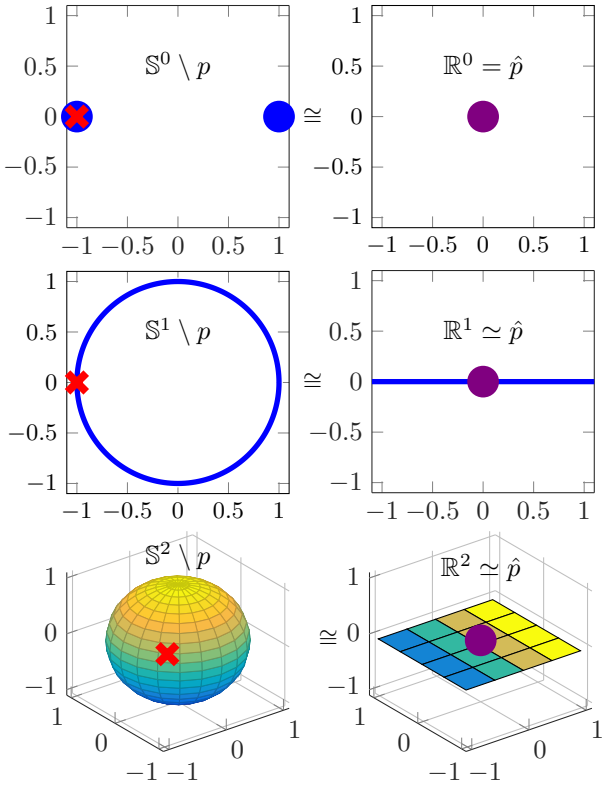}
	\caption{Visualization of $\mathbb{S}^n\setminus p$ (left) and $\mathbb{R}^n$ (right) for \\ $n=0,1,2$. The red cross marks a point $p$ removed from the unit sphere. Each space to the left is homeomorphic to the corressponding space to the right, \emph{i.e.}, $\mathbb{S}^n\setminus p \cong \mathbb{R}^n$. In turn, $\mathbb{R}^n$ is homotopy equivalent to a point (for instance $\hat{p}$ marked by a purple dot in each plot to the right) and therefore $\mathbb{S}^n\setminus p$ is contractible according to \cref{lm:contract_homeomorphism}. Higher dimensions are difficult to visualize, and therefore $\mathbb{S}^2$ is commonly used to visualize parts of the quaternion set, as done in \cref{fig:result_quat_spheres}.} \label{fig:sphere_vs_plane}
\end{figure}

\begin{theorem}
\label{thm:quat_contractible}
The set of unit quaternions $\mathbb{H}$ minus a point $\tilde{q} \in \mathbb{H}$, denoted $\mathbb{H}\setminus{\tilde{q}}$, is contractible. 
\end{theorem}
\begin{proof}
The set $\mathbb{H}$ is homeomorphic to $\mathbb{S}^3$ \cite{lavalle2006planning}. Therefore $\mathbb{H}\setminus{\tilde{q}} \cong \mathbb{S}^3\setminus{p}$ for some point $p \in \mathbb{S}^3$, according to \cref{lm:remove_point}. \cref{thm:sphere_n} with $n=3$ yields that $\mathbb{S}^3\setminus{p}$ is contractible, and because of the homeomorphic relation, \cref{lm:contract_homeomorphism} yields that $\mathbb{H}\setminus{\tilde{q}}$ is also contractible.
\end{proof}

It is noteworthy that the contractible subset $\mathbb{H}\setminus{\tilde{q}}$ is the largest possible subset of $\mathbb{H}$, because one point is the smallest possible subset to remove. Hence, it is guaranteed that no unnecessary restriction is made in \cref{thm:quat_contractible}, though there are other, more limited, subsets of $\mathbb{H}$ that are also contractible. Sometimes only half of $\mathbb{H}$, for instance the upper half of the quaternion hypersphere, is used to represent orientations. However, instead of continuous transitions between the half spheres this results in discontinuities within the upper half sphere \cite{lavalle2006planning}. In the context of DMPs and automatic control such discontinuities would cause severe obstructions, which motivates the search for the largest possible contractible subset of $\mathbb{H}$. One of the experiments (Setup~3 in \cref{sec:experiments}) provides an example of when both half spheres are necessary for a continuous representation of the robot orientation.

\section{Control Algorithm}
\label{sec:control_algorithm}
In this section, we augment the controller in \cite{karlsson2017dmp,karlsson2018convergence} to incorporate orientation in Cartesian space. The resulting algorithm can also be seen as a temporally coupled version of the Cartesian DMPs proposed in \cite{ude2014orientation}. The pose in Cartesian space consists of position and orientation. The position control in this paper is the same as described in \cite{karlsson2017dmp,karlsson2018convergence}, except that it is also affected by the orientation through the shared time parameter $\tau_a$ in this paper.

Similar to the approaches in \cite{ude2014orientation,ude1999filtering}, we define a difference between two quaternions, $q_1$ and $q_2$, as 
\begin{equation}
\label{eq:quat_diff}
d(q_1\bar{q}_2) = 2 \cdot \text{Im}[ \log(q_1\bar{q}_2)] \in \mathfrak{h}
\end{equation}
where $\mathfrak{h}$ is the orientation difference space, defined as the image of $d$, and Im denotes the imaginary quaternion part, assuming for now that $q_1\bar{q}_2 \neq (-1,0,0,0)$. 
This is elaborated on in \cref{sec:discussion}. Further, we will use a shorter notation, so that for instance 
\begin{equation}
d_{cg} = d(q_c\bar{q}_g) = 2 \cdot \text{Im}[ \log(q_c\bar{q}_g)]
\end{equation}
represents the difference between coupled and goal orientations. This mapping preserves the contractibility concluded in \cref{sec:contractible}, as established by \cref{thm:h_contractible}.
\begin{theorem}
\label{thm:h_contractible}
The orientation difference space $\mathfrak{h}$ is contractible.
\end{theorem}

\begin{proof}
The mapping 
\begin{equation}
d \hspace{2mm} : \hspace{2mm} \mathbb{H}\setminus (-1,0,0,0) \rightarrow \mathfrak{h}
\end{equation}
has the properties necessary to qualify as a homeomorphism. It is one-to-one \cite{ude1999filtering} and onto, continuous (since the point $(-1,0,0,0)$ has been removed), and its inverse (division by 2 followed by the exponential map) is also continuous. Further, its domain $\mathbb{H}\setminus (-1,0,0,0)$ is contractible (see \cref{thm:quat_contractible}), and therefore its image $\mathfrak{h}$ is contractible (see \cref{lm:contract_homeomorphism}).
\end{proof}

Using the function $d$, a coupled DMP pose trajectory is modeled by the dynamical system
\begin{align}
\label{eq:dotz}
\tau_a \dot{z} &= \alpha_z(\beta_z(g-y_c)-z) + f_p(x) \\
\tau_a \dot{y}_c &= z  \\
\tau_a \dot{\omega}_z &= \alpha_z(\beta_z (-d_{cg})-\omega_z) + f_o(x) \\
\tau_a \omega_c &= \omega_z
\end{align}
Here, $x$ is a phase variable that evolves as 
\begin{align}
\tau_a \dot{x} =& -\alpha_x x
\label{eq:x}
\end{align}
Further, $f_o(x)$ is a virtual forcing term in the orientation domain, and each element $i$ of $f_o(x)$ is given by
\begin{align}
f^i_o(x) =& \frac{\sum_{j=1}^{N_b} \Psi_{i,j}(x)w_{i,j}}{\sum_{j=1}^{N_b} \Psi_{i,j}(x)} x \cdot d_i(q_g \bar{q}_0)
\label{eq:f}
\end{align}
where each basis function, $\Psi_{i,j}(x)$, is determined as
\begin{align}
\Psi_{i,j}(x) =& \exp \left(-\frac{1}{2\sigma_{i,j}^2}(x-c_{i,j})^2 \right)
\label{eq:psi}
\end{align}
Here, $\sigma$ and $c$ denote the width and center of each basis function, respectively. The forcing term $f_p(x)$ is determined accordingly, see \cite{karlsson2017dmp,karlsson2018convergence}. Further, the parameters of $f(x)$ can be determined based on a demonstrated trajectory by means of locally weighted regression \cite{atkeson1997locally}, as described in \cite{ijspeert2013dynamical}. 

All dimensions of the robot pose are temporally coupled through the shared adaptive time parameter $\tau_a$. Denote by $y_a$ the actual position of the robot, and by $q_a$ the actual orientation. The adaptive time parameter $\tau_a$ is determined based on the low-pass filtered difference between the actual and coupled poses as follows.
\begin{align}
\dot{e}_p &= \alpha_e(y_a - y_c - e_p) \\
\dot{e}_o &= \alpha_e(d_{ac} - e_o)  \\
\label{eq:edot} 
e &= [e_p^T \hspace{2mm}  e_o^T]^T \\
\tau_a &= \tau(1 + k_c e^Te)
\label{eq:taua}
\end{align}
This causes the evolution of the coupled system to slow down in case of configuration deviation; see \cite{ijspeert2013dynamical,karlsson2017dmp}. Moreover, the controller below is used to drive $y_a$ to $y_c$, and $q_a$ to $q_c$.
\begin{align}
\label{eq:our_ddoty}
\ddot{y}_r &= k_p(y_c-y_a) + k_v(\dot{y}_c - \dot{y}_a) + \ddot{y}_c \\
\dot{\omega}_r &= -k_p d_{ac} - k_v(\omega_a-\omega_c)+\dot{\omega}_c
\label{eq:our_dotomega}
\end{align}
This can be seen as a pose PD controller together with the feedforward terms $\ddot{y}_c$ and $\dot{\omega}_c$. Here, $\ddot{y}_r$ and $\dot{\omega}_r$ denote reference accelerations sent to the internal controller of the robot, after conversion to joint values using the robot Jacobian \cite{spong2006robot}. We let $k_p = k_v^2/4$, so that \cref{eq:our_ddoty}~--~\cref{eq:our_dotomega} represent a critically damped control loop. Similarly, $\beta_z=\alpha_z / 4$ \cite{ijspeert2002humanoid}. The control system is schematically visualized in \cref{fig:coupling_scheme}. We model the 'Robot' block as a double integrator, so that $\ddot{y}_a=\ddot{y}_r$ and $\dot{\omega}_a= \dot{\omega}_r$, as justified in \cite{karlsson2018convergence} for accelerations with moderate magnitudes and changing rates. In summary, the proposed control system is given by
\begin{align}
\label{eq:cartesian_entire_control_a}
\ddot{y} &= -k_\text{p}(y-y_\text{c}) - k_\text{v} (\dot{y}-\dot{y}_\text{c}) + \ddot{y}_\text{c} \\
\label{eq:cartesian_entire_control_b}
\dot{\omega}_\text{a} &= -k_\text{p} d_\text{ac} - k_\text{v}(\omega_\text{a}-\omega_\text{c})+\dot{\omega}_\text{c} \\
\label{eq:cartesian_entire_control_c}
\dot{e} &= \alpha_e\left(\left[ \left[ y-y_\text{c} \right]^T \phantom{d} d_\text{ac}^T\right]^T - e\right) \\
\tau_\text{a} &= \tau (1 + k_\text{c} e^Te) \\
\tau_\text{a} \dot{x} &= -\alpha_x x \\
\tau_\text{a} \dot{y}_\text{c} &= z  \\
\tau_\text{a} \dot{z} &= \alpha(\beta(g-y_\text{c})-z) + f_\text{p}(x) \\
\tau_\text{a} \omega_\text{c} &= \omega_z \\
\label{eq:cartesian_entire_control_last}
\tau_\text{a} \dot{\omega}_z &= \alpha(\beta(-d_{\text{c}g})-\omega_z) + f_\text{o}(x)
\end{align}
We introduce a state vector $\xi$ as
\begin{equation}
\xi = 
\begin{pmatrix}
y - y_\text{c} \\[3pt] \dot{y}-\dot{y}_\text{c} \\[3pt] d_\text{ac} \\[3pt] \omega_\text{a}-\omega_\text{c} \\[3pt] e \\[3pt] x \\[3pt] y_\text{c}-g \\[3pt] z \\[3pt]d_{\text{c}g} \\[3pt] w_z
\end{pmatrix}
\in \mathbb{R}^{22} \times \mathfrak{h}^3
\label{eq:cartesian_states_first_time}
\end{equation}

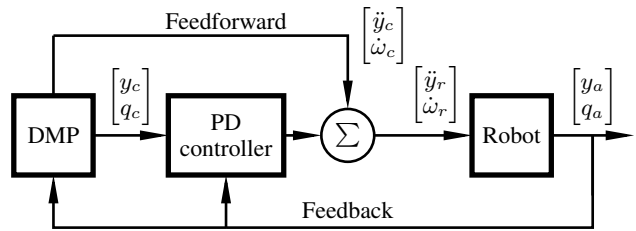
\begin{figure}
	\centering
    \tikzset{block/.style={draw, rectangle, line width=2pt,
     minimum height=3em, minimum width=3em, outer sep=0pt}}
\tikzset{sumcircle/.style={draw, circle, outer sep=0pt, 
     label=center:{{$\sum$}}, minimum width=2em}}
\tikzset{every picture/.style={auto, line width=1pt,
          >=narrow,font=\small}}

\begin{tikzpicture}
\node[block](Robot){Robot};
\node[sumcircle, left=13mm of Robot](sum3){};
\node[block, left=5mm of sum3,text width=1.3cm,align=center](PD){PD controller};
\node[block, left=10mm of PD](DMP){DMP};
\coordinate[right=11mm of Robot](c1);
\draw[->](DMP)--node[](yc){$\begin{bmatrix} y_c \\ q_c \end{bmatrix}$}(PD);
\draw[->](PD)--(sum3);
\draw[->](sum3)--node[](yardd){ }(Robot);
\draw[->](Robot)--node[](ya){$\begin{bmatrix} y_a \\ q_a \end{bmatrix}$}(c1);

\node[above left=-5mm and 0mm of Robot](tt){$\begin{bmatrix} \ddot{y}_{r} \\ \dot{\omega}_r \end{bmatrix}$};

\node[below=4mm of sum3](tt2){Feedback};
\node[above=7.5mm of PD](tt2){Feedforward};

\coordinate[below=8mm of ya](c2);
\coordinate[below=7mm of PD](c3);
\coordinate[above=8mm of DMP](c4);

\draw[-](DMP)--(c4);
\draw[->](c4)-|node[](ycdd){$\begin{bmatrix} \ddot{y}_c \\ \dot{\omega}_c \end{bmatrix}$}(sum3);
\draw[-](ya)--(c2);
\draw[-](ya)|-(c3);
\draw[->](c3)-|(DMP);
\draw[->](c3)-|(PD);
\end{tikzpicture}
    \caption{The control structure for temporally coupled Cartesian DMPs. The block denoted 'Robot' includes the internal controller of the robot, together with transformations between Cartesian and joint space for low-level control. The 'DMP' block corresponds to the computations in \cref{eq:dotz}~--~\cref{eq:taua}. The PD controller and the feedforward terms are specified in \cref{eq:our_ddoty}~--~\cref{eq:our_dotomega}. This forms a cascade controller, with the DMP as outer controller and the PD as the inner.}
\label{fig:coupling_scheme}
\end{figure}

\section{Experiments}
\label{sec:experiments}
The control law in \cref{sec:control_algorithm} was implemented in the Julia programming language \cite{BEKS14}, to control an ABB YuMi \cite{yumi} robot. The Julia program communicated with the internal robot controller through a research interface version of Externally Guided Motion (EGM) \cite{egm,bagge2017yumi} at a sampling rate of \SI{250}{Hz}. 

Three different setups were used to investigate the behavior of the controller. As preparation for each setup, a temporally coupled Cartesian DMP had been determined from a demonstration by means of lead-through programming, which was available in the YuMi product by default. In each trial, the temporally coupled DMP was executed while the magnitudes of the states in \cref{eq:cartesian_states_first_time} were logged.

Perturbations were introduced by physical contact with a human. This was enabled by estimating joint torques induced by the contact, and mapping these to Cartesian contact forces and torques using the robot Jacobian. A corresponding acceleration was then added to the reference acceleration $\ddot{y}_r$ as a load disturbance. However, we emphasize that this paper is not focused on how to generate the perturbations themselves. Instead, that functionality was used only as an example of unforeseen deviations, and to investigate the stability properties of the proposed control algorithm.

A video of the experimental arrangement is available as an attachment to this paper, and a version with higher resolution is available in \cite{cartesian_dmp_youtube}. The setups were as follows.

\begin{figure}
\centering
\begin{minipage}{.48\columnwidth}
\centering
\includegraphics[width=\columnwidth,height=28mm]{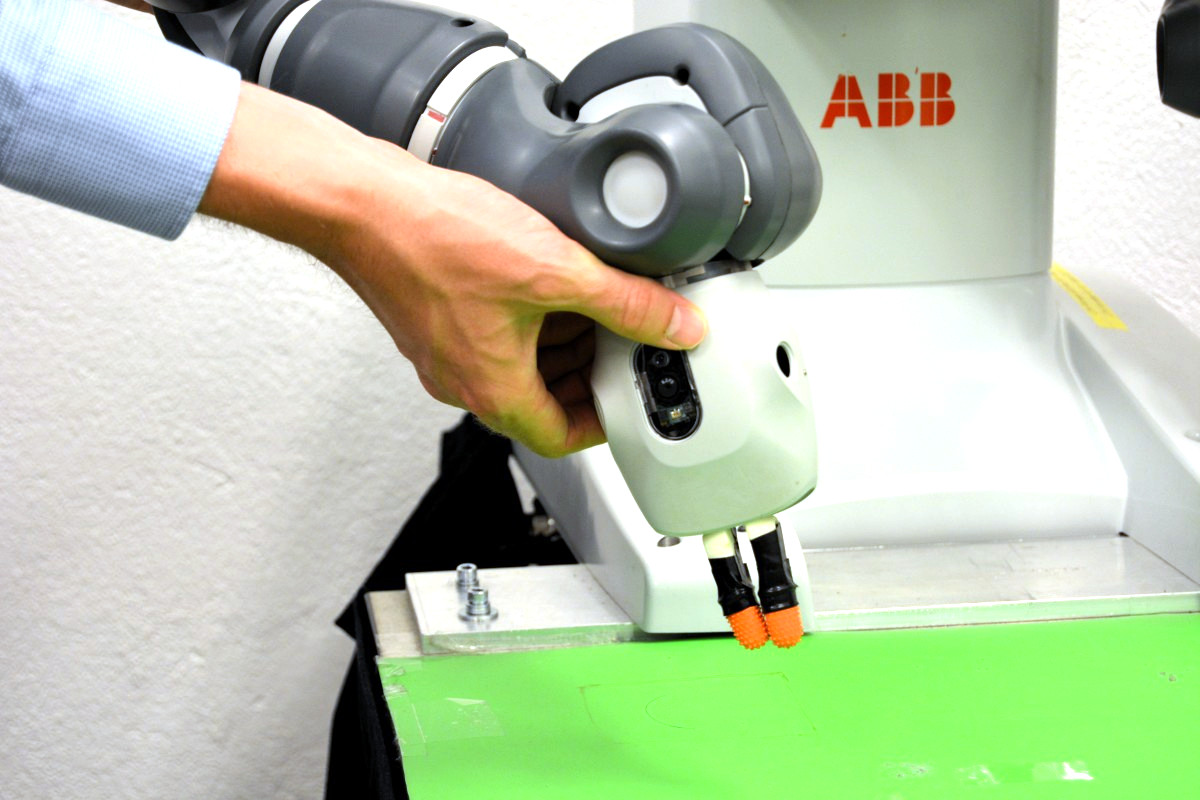}
\subcaption{}
\end{minipage}\hfill
\begin{minipage}{.48\columnwidth}
\includegraphics[width=\columnwidth]{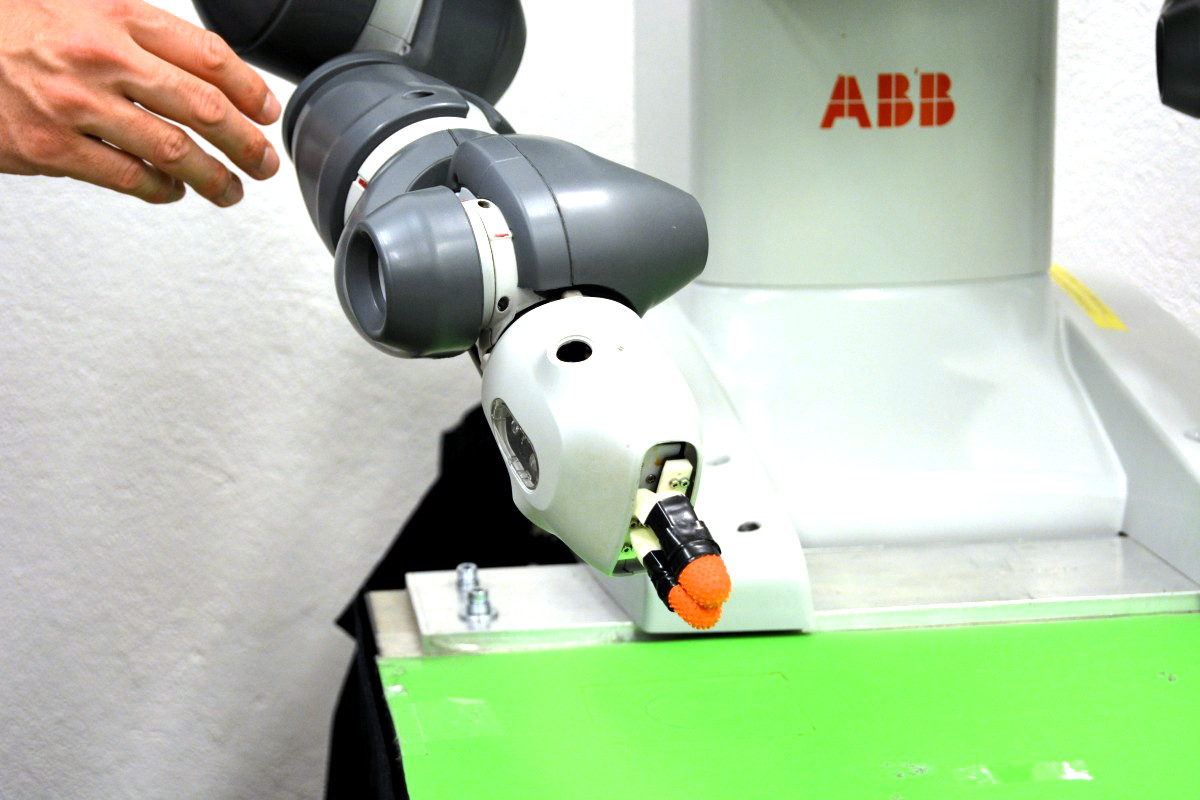}
\subcaption{}
\end{minipage}
\caption{Photographs of a trial of Setup 1. The robot was initially released from the pose in (a), with an offset to the goal pose. In (b), the goal pose was reached.}
\label{fig:setup1}
\end{figure}

\textbf{Setup 1.} This setup is visualized in \cref{fig:setup1}. Prior to the experiment, a test DMP that did not perform any particular task was executed, and the robot then converged to the goal pose, \emph{i.e.}, to $y_a=y_c=g$ and $d_{ac}=d_{cg}=0$. Thereafter, the operator pushed the end-effector, so that the actual pose deviated from the coupled and goal poses. The experiment was initialized when the operator released the robot arm. The purpose of this procedure was to examine the stability of the subsystem in (\ref{eq:cartesian_entire_control_a})~--~(\ref{eq:cartesian_entire_control_c}). A total of 100 perturbations were conducted.

\textbf{Setup 2.} See \cref{fig:setup2}. The task of the robot was to reach a work object (in this case a gore-tex graft used in cardiac and vascular surgery) from its home position. A DMP defined for this purpose was executed, and the operator introduced two perturbations during the robot movement. The purpose of this setup was to investigate the stability of the entire control system in (\ref{eq:cartesian_entire_control_a})~--~(\ref{eq:cartesian_entire_control_last}). A total of 10 trials were conducted.

\textbf{Setup 3.} See \cref{fig:setup3}. The task of the robot was to hand over the work object from its right arm to its left.  The movement was specifically designed to require an end-effector rotation angle of more than $\pi$, thus requiring both the upper and the lower halves of the quaternion hypersphere (see \cref{fig:result_quat_spheres}), and not only one of the halves which is sometimes used \cite{lavalle2006planning}. Such movements motivate the search for the largest possible contractible subset of $\mathbb{H}$ in \cref{sec:contractible}. Similar to Setup~2, the purpose was to investigate the stability of (\ref{eq:cartesian_entire_control_a})~--~(\ref{eq:cartesian_entire_control_last}), and 10 trials were conducted.

\begin{figure}
\centering
\begin{minipage}{.48\columnwidth}
\centering
\includegraphics[width=\columnwidth,height=28mm]{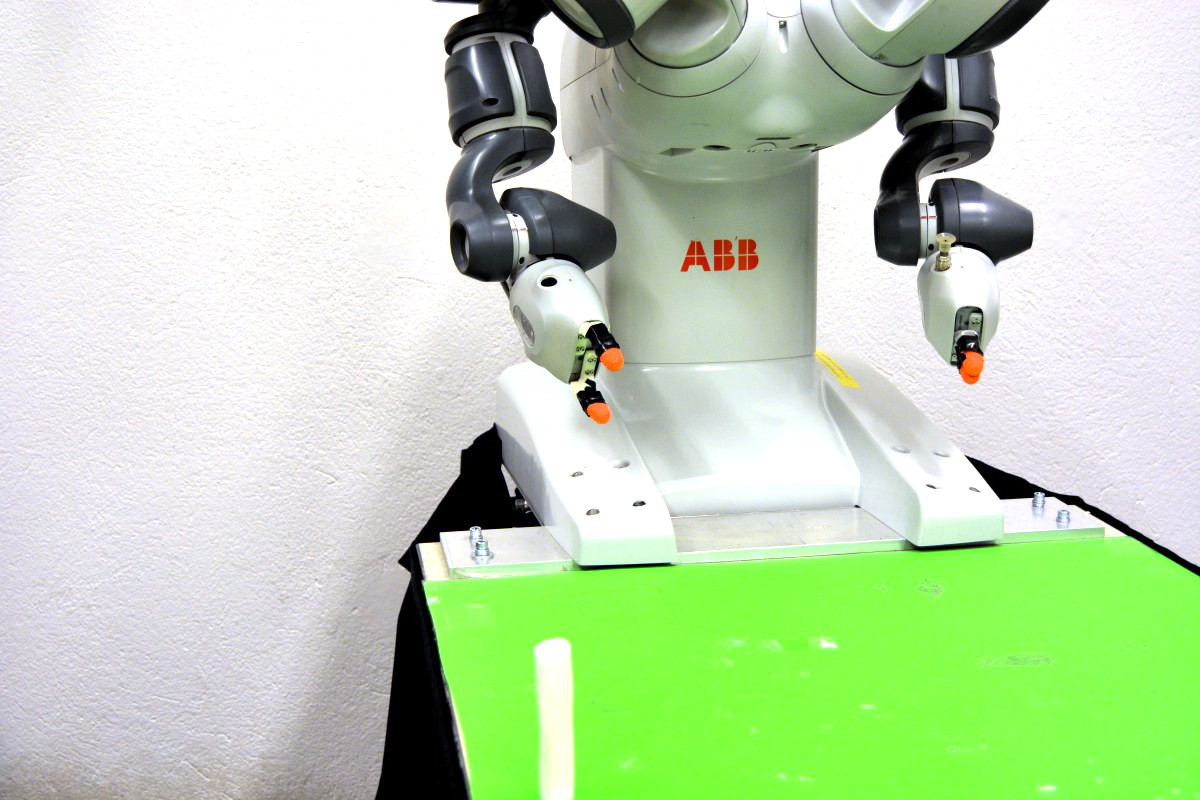}
\subcaption{}
\end{minipage}\hfill
\begin{minipage}{.48\columnwidth}
\includegraphics[width=\columnwidth]{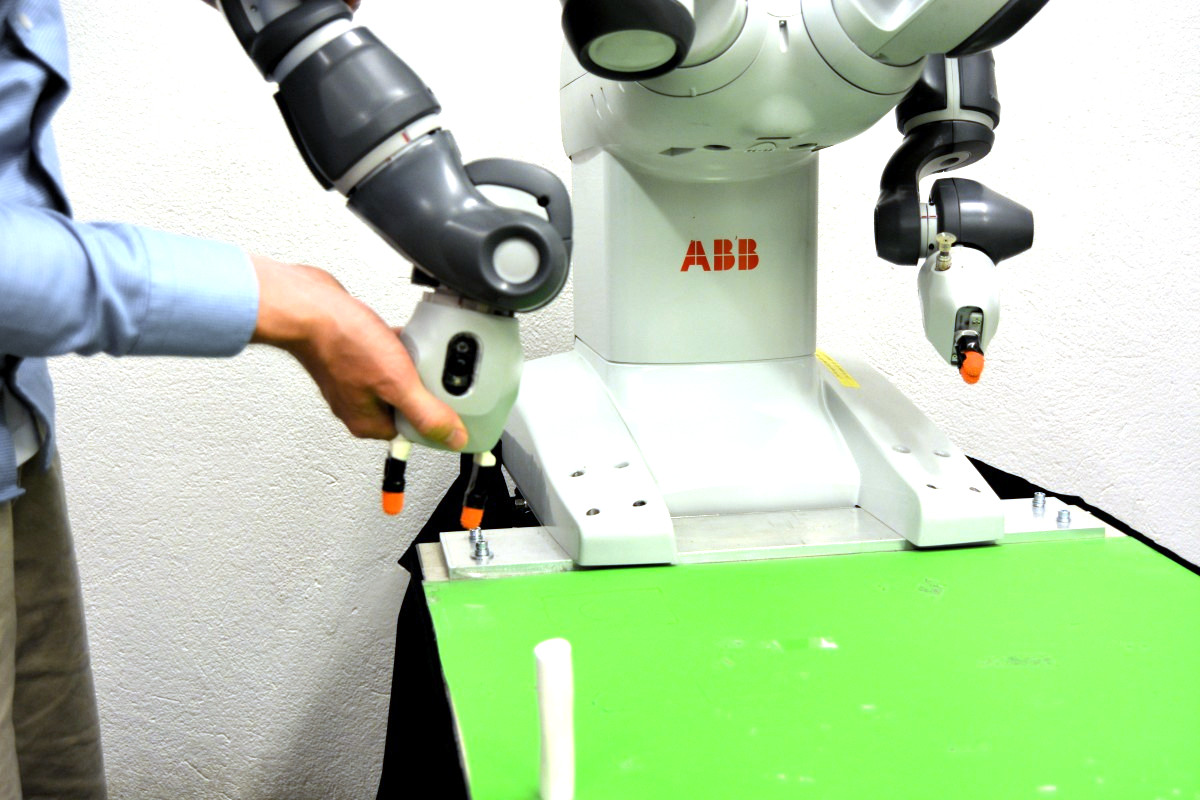}
\subcaption{}
\end{minipage}
\par\vspace{3mm} 
\begin{minipage}{.48\columnwidth}
\includegraphics[width=\columnwidth]{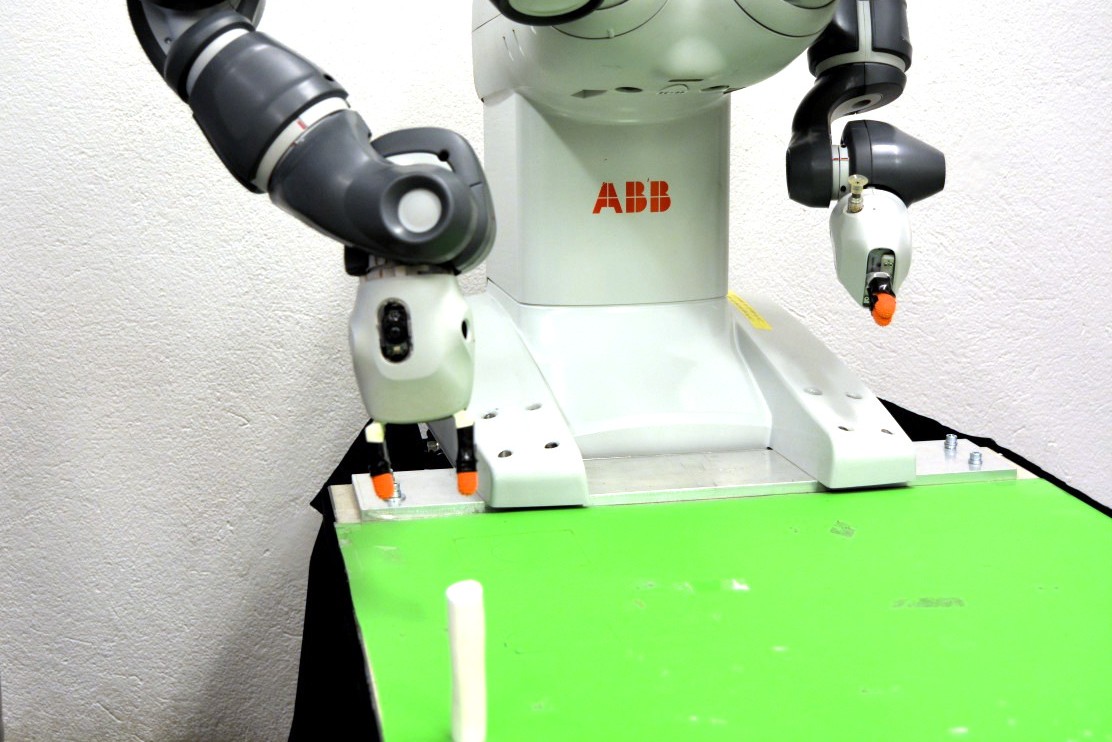}
\subcaption{}
\end{minipage}\hfill
\begin{minipage}{.48\columnwidth}
\centering
\includegraphics[width=\columnwidth]{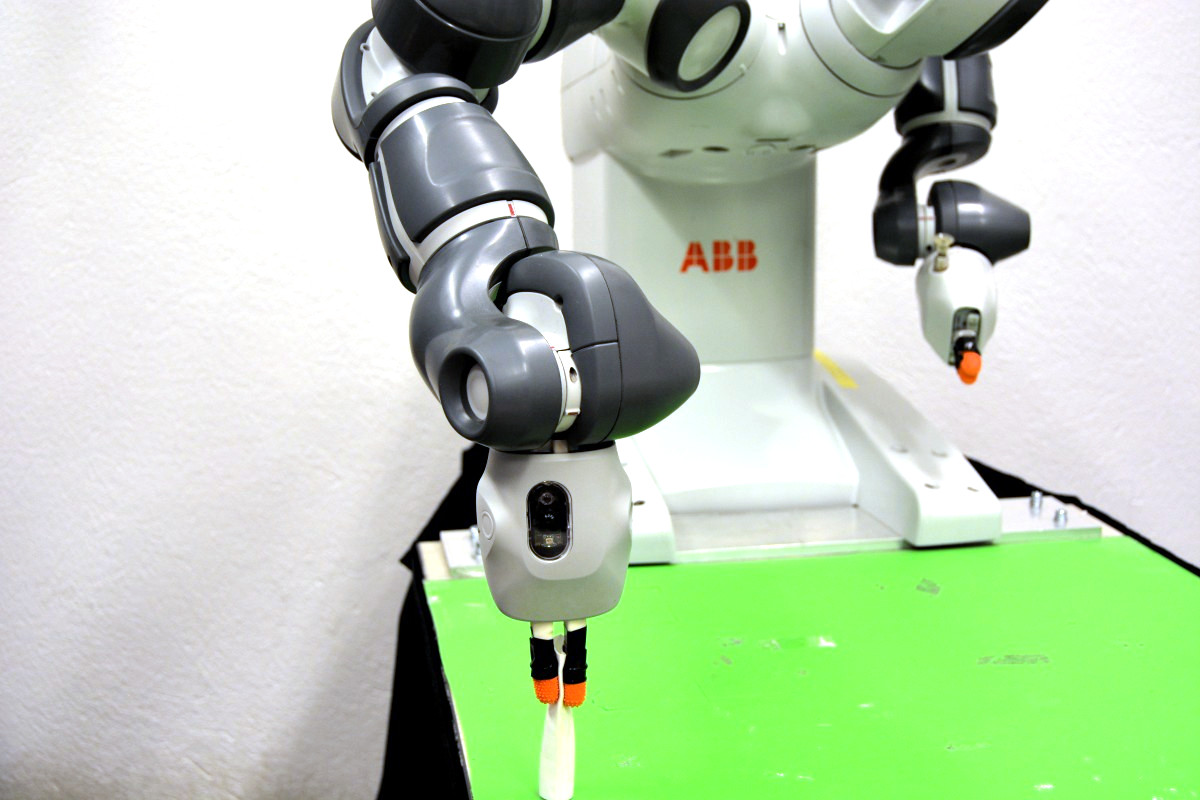}
\subcaption{}
\end{minipage}
\caption{Photographs of a trial of Setup 2. The DMP was executed from the home position (a), and was perturbed twice on its way toward the goal (b). It recovered from these perturbations (c), and reached the goal at the work object (d).}
\label{fig:setup2}
\end{figure}

\begin{figure}
\centering
\begin{minipage}{.48\columnwidth}
\centering
\begin{tikzpicture}
\node[anchor=south west,inner sep=0] at (0,0) {\includegraphics[width=\columnwidth,height=28mm]{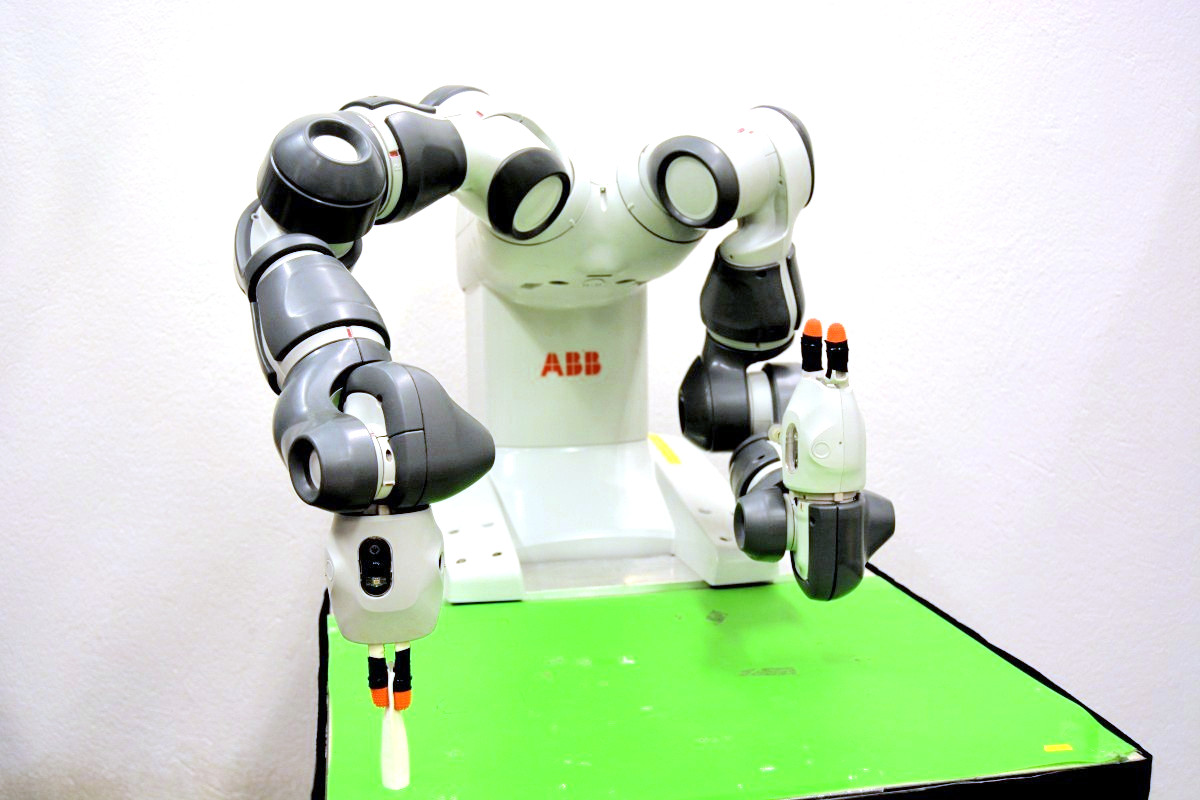}};
\draw [->,line width=1.5pt,red] (1.4,.5) arc[x radius=.5cm, y radius =.5cm, start angle=270, end angle=90];
\end{tikzpicture}
\subcaption{}
\end{minipage}\hfill
\begin{minipage}{.48\columnwidth}
\includegraphics[width=\columnwidth]{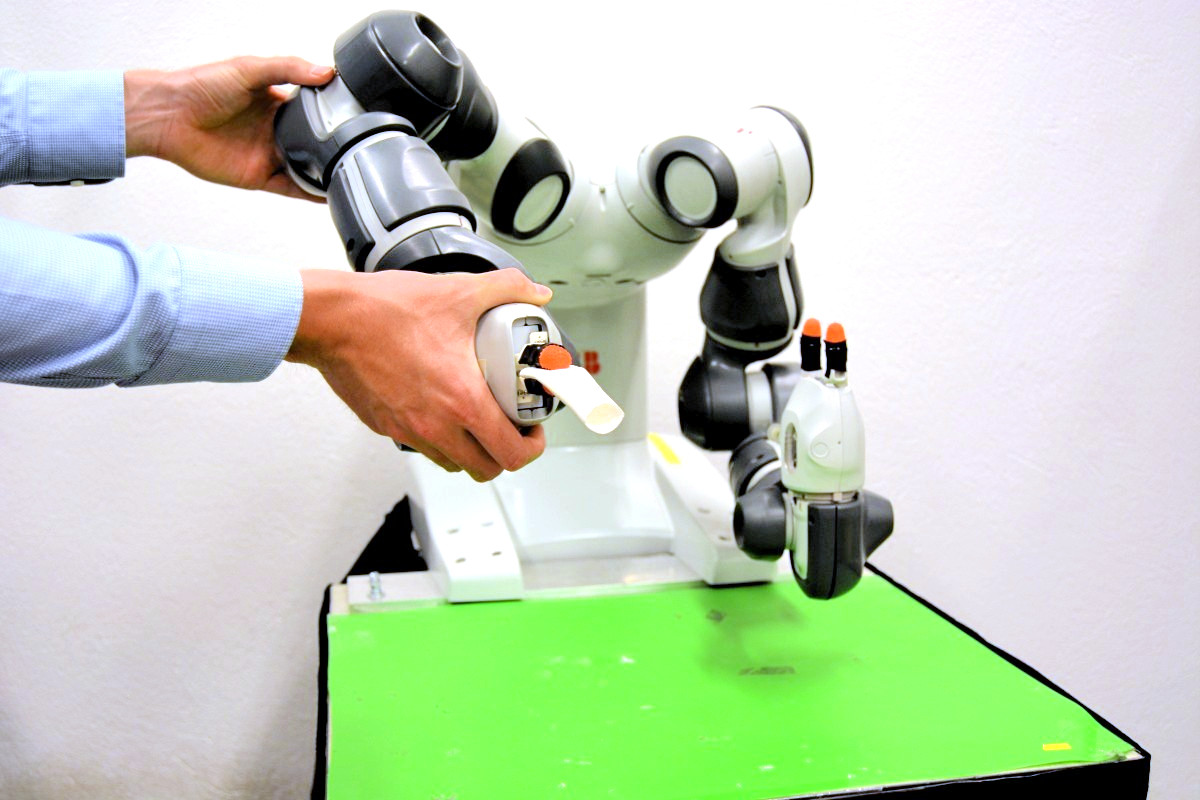}
\subcaption{}
\end{minipage}
\par\vspace{3mm} 
\begin{minipage}{.48\columnwidth}
\begin{tikzpicture}
\node[anchor=south west,inner sep=0] at (0,0) {\includegraphics[width=\columnwidth,height=28mm]{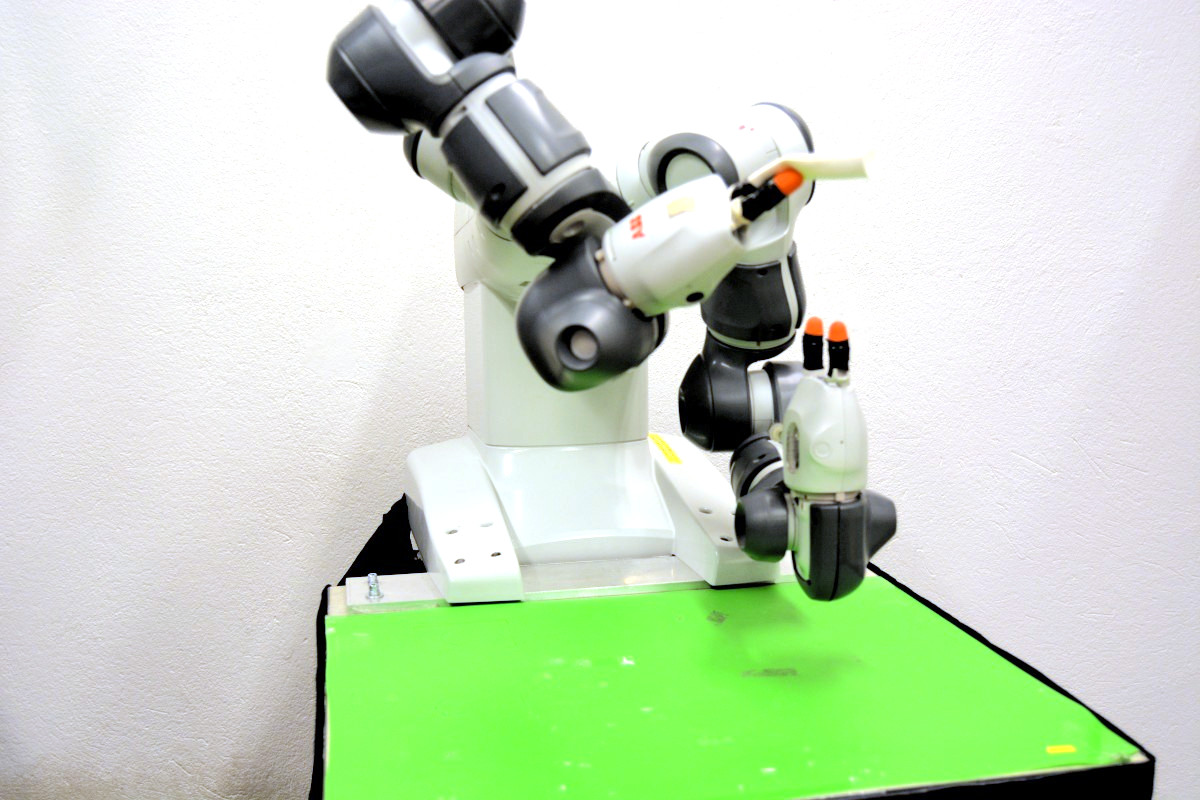}};
\draw [->,line width=1.5pt,red] (2,2.3) arc[x radius=0.6cm, y radius =.6cm, start angle=90, end angle=0];
\end{tikzpicture}\subcaption{}
\end{minipage}\hfill
\begin{minipage}{.48\columnwidth}
\centering
\includegraphics[width=\columnwidth]{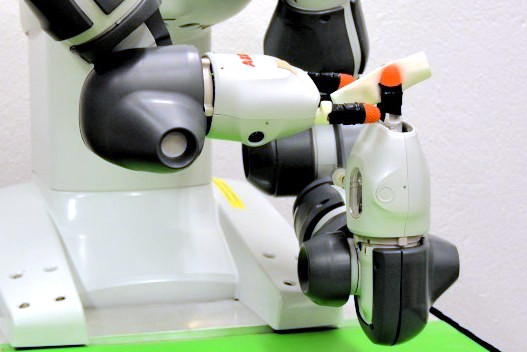}
\subcaption{}
\end{minipage}
\caption{Photographs of a trial of Setup 3. The robot started its movement from the configuration in (a). The end-effector was rotated as indicated by the red arrows, which resulted in a rotation larger than $\pi$ from start to goal. The robot was perturbed twice by the operator (b), recovered and continued its movement (c), and accomplished the handover (d).}
\label{fig:setup3}
\end{figure}

\begin{figure}[h]
	\centering	
	\input{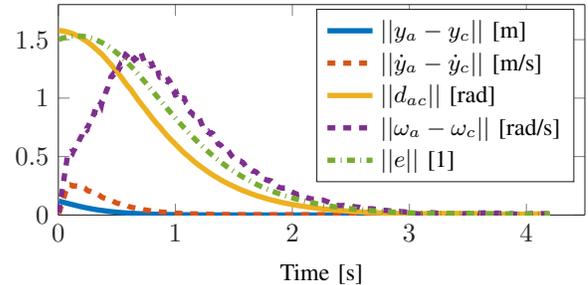}
	\caption{Data from a trial of Setup 1. The notation $\Vert\cdot\Vert$ represents the 2-norm, and the unit symbol [1] indicates dimensionless quantity. The experiment was initialized with some position error $y_a-y_c$ and orientation error $d_{ac}$. The operator released the robot at $t=0$. It can be seen that each state converged to 0.}
	\label{fig:setup1_plot}
\end{figure}

\section{Results}
\label{sec:results}
Figures~\ref{fig:setup1_plot}--\ref{fig:result_quat_spheres} display data from the experiments. \Cref{fig:setup1_plot} shows the magnitude of the states during a trial of Setup~1, and it can be seen that each state converged to 0 after the robot had been released. Similarly, \cref{fig:grasp_states,fig:handover_states} show data from Setup~2 and 3 respectively, and it can be seen that the robot recovered from each of the perturbations. Further, each state subsequently converged to 0. All trials in a given setup gave similar results. Further, these results suggest that the control system (\ref{eq:cartesian_entire_control_a})~--~(\ref{eq:cartesian_entire_control_last}) is exponentially stable.

\begin{figure}
	\centering	
	\input{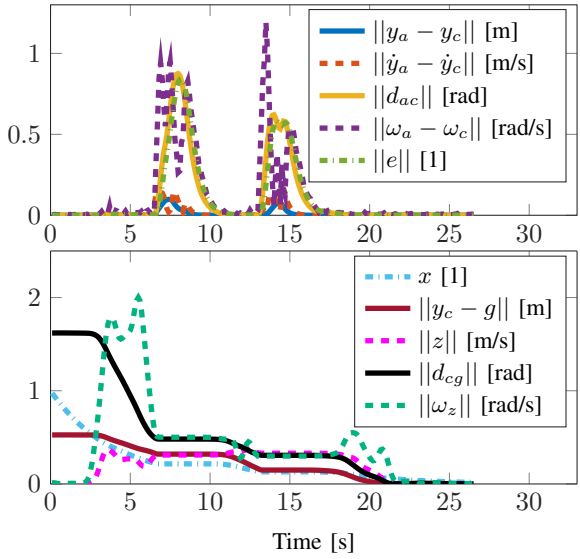}
	\caption{Data from a trial of Setup 2. Consider first the upper plot. The two perturbations are clearly visible, and these were recovered from as the states converged to 0. In the lower plot, it can be seen that the time evolution of the states was slowed down in the presence of perturbations. It can further be seen that each of the states converged to 0.}
	\label{fig:grasp_states}
\end{figure}

\Cref{fig:result_quat_spheres} shows orientation data from Setup~2 (left) and Setup~3 (right). The upper plots show quaternions for the demonstrated paths, $q_d$, determined using lead-through programming prior to the experimental trials, relative to the goal quaternions $q_g$. The middle plots show coupled orientations $q_c$ relative to $q_g$. It can be seen that the paths of $q_d$ and $q_c$ were similar for each of the setups, which was expected given a sufficient number of DMP basis functions. The perturbations can be seen in the bottom plots, which show $q_a$ relative to $q_c$. Though $q_a \bar{q}_c$ was very close to the identity quaternion for most of the time, it deviated significantly twice per trial as a result of the perturbations. Setup~3 is an example of a movement where it would not be possible to restrict the quaternions to the upper half sphere, without introducing discontinuities. This is shown in \Cref{fig:result_quat_spheres}, as quaternions were present not only on the upper half sphere, but also on the lower, for Setup~3.

\section{Discussion}
\label{sec:discussion}
In each of the experiments, the robot recovered from the perturbations and subsequently reached the goal pose, which was the desired behavior. Further, the behavior corresponded to that in \cite{karlsson2017dmp,karlsson2018convergence}, except that orientations in Cartesian space are now supported. Most of the discussion in \cite{karlsson2017dmp,karlsson2018convergence} is therefore valid also for these results, and is not repeated here.

A mathematical proof that the proposed control system is exponentially stable would enhance the contribution of this paper, but remains as future research. Nevertheless, is has now been shown that the topology of $\mathfrak{h}$ does not prohibit a globally exponentially stable control system. One may object that this topological result is not directly necessary for the control design in \cref{sec:control_algorithm}. However, it is still useful because it rules out the otherwise possible obstruction of a non-contractible state space. This result is relevant not only for DMP applications, but for any control application where quaternions are used to represent orientation. Furthermore, the experimental results indicate exponential stability, since in practice the DMP states converged to 0.

The magnitude of the difference between two quaternions, $||d(q_1 \bar{q}_2)||$, corresponds to the length of a geodesic curve connecting $q_1$ and $q_2$ \cite{ude1999filtering}. This results in proper scaling between orientation difference and angular velocity in the DMP control algorithm, as explained in \cite{ude2014orientation}. This is the reason why the quaternion difference in \cref{eq:quat_diff} was used in \cite{ude2014orientation} and in this paper. 

In \cref{sec:contractible}, the largest possible contractible subset of $\mathbb{H}$ was found as $\mathbb{H} \setminus \tilde{q}$. Hence, it is not necessary to remove a large proportion of the quaternion set, which is sometimes done. For instance, sometimes the lower half of the quaternion hypersphere is removed \cite{lavalle2006planning}, which is unnecessarily limiting. The results from Setup~3 show that this proposed method works also when it is necessary to use both half spheres, see \cref{fig:result_quat_spheres}. In \cref{sec:control_algorithm}, the removed point $\tilde{q}$ was chosen as $(-1,0,0,0)$, which corresponds to a full $2 \pi$ rotation from the identity quaternion. A natural question is therefore how to handle the case where $(-1,0,0,0)$ is visited by $q_a \bar{q}_c$ or $q_c \bar{q}_g$. In theory, almost any control signal could be used to move the orientations away from this point, and in practice a single point would never be visited because it is infinitely small. However, in practice some care should be taken in a small region around $(-1,0,0,0)$, because of possible numerical difficulties and rapidly changing control signals.

In this paper, the same control gains were used in the position domain as in the orientation domain. This was done in order to limit the notation, but is not actually required.

An interesting direction of future work is to use the proposed controller to warm start reinforcement learning approaches for robotic manipulation. Reinforcement learning with earlier DMP versions has been investigated in, \textit{e.g.}, \cite{stulp2011learning,stulp2012reinforcement,stulp2012model,li2018reinforcement}.

\section{Conclusion}
In this paper, it was first shown that the unit quaternion set minus one point is contractible, thus allowing for continuous and asymptotically stable control systems. This was used to design a control algorithm for DMPs with temporal coupling in Cartesian space. The proposed DMP functionality was verified experimentally on an industrial robot.

A video that shows the experiments is provided as an attachment to this paper, and a version with higher resolution is available in \cite{cartesian_dmp_youtube}.

\begin{figure}[h]
	\centering	
	\input{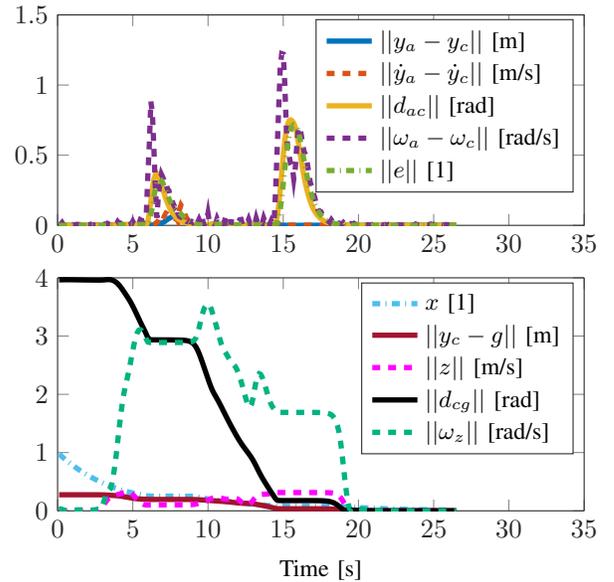}
	\caption{Data from a trial of Setup 3. The organization is the same as in \cref{fig:grasp_states}, and similar conclusions can be drawn. In addition, the required rotation angle from start to goal was larger than $\pi$ in this setup, which corresponds to $||d_{cg}||$ being larger than $\pi$ initially.}
	\label{fig:handover_states}
\end{figure}

\begin{figure}
	\centering	
	\input{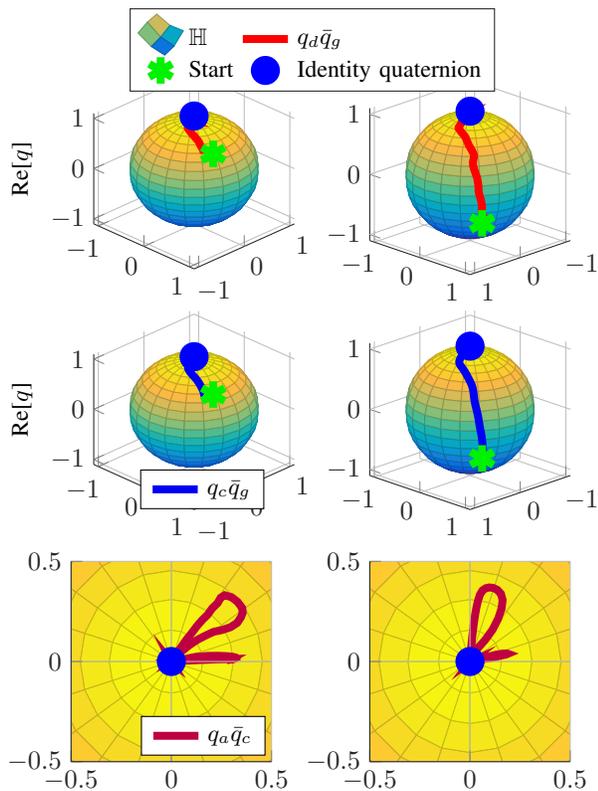}
	\caption{Orientation data from Setup~2 (left) and Setup~3 (right). Quaternions have been projected on $\mathbb{S}^2$ for the purpose of visualization. Vertical axes represent quaternion real parts, and horizontal axes represent the first two imaginary elements with magnitudes adjusted to yield unit length of the resulting projection. The bottom plots show the quaternion set seen from above, and hence their real axes are directed out from the figure.}
	\label{fig:result_quat_spheres}
\end{figure}

\bibliographystyle{IEEEtran}
\bibliography{cartesian_dmps}
\end{document}